\documentclass{article}

\usepackage{amssymb}
\usepackage{amsmath,amsthm}
\usepackage{amsfonts}
\usepackage{url}
\usepackage{graphicx}
\usepackage{fullpage}

\usepackage{natbib}
\usepackage{epstopdf}

\usepackage[utf8]{inputenc}   % allow utf-8 input
\usepackage[T1]{fontenc}      % use 8-bit T1 fonts
\usepackage{hyperref}           % hyperlinks
\usepackage{url}                    % simple URL typesetting
\usepackage{booktabs}         % professional-quality tables
\usepackage{amsfonts}         % blackboard math symbols
\usepackage{nicefrac}           % compact symbols for 1/2, etc.
\usepackage{microtype}        % microtypography

\newenvironment{myproof}[2] {\noindent \emph{Proof of {#1} {#2}.}}

\newtheorem{theorem}{Theorem}
\newtheorem{lemma}[theorem]{Lemma}

\newtheorem{corollary}[theorem]{Corollary}
\newtheorem{proposition}[theorem]{Proposition}

\renewcommand{\P}{\mathbb{P}}
\newcommand{\E}{\mathbb{E}}
\newcommand{\tru}[1]{{#1}^\star}
\newcommand{\emp}[1]{\hat{#1}}

\renewcommand{\H}{\calH}
\newcommand{\X}{\calX}
\newcommand{\C}{\calC}
\newcommand{\Y}{\calY}
\newcommand{\bs}{\boldsymbol}
\renewcommand{\vec}[1]{\bs{\mathrm{#1}}}

\newcommand{\hide}[1]{}
\newcommand{\oo}[1]{\frac{1}{#1}}
\usepackage{bbm}
\newcommand{\chr}{\boldsymbol{\mathbbm{1}}} %
\newcommand{\pred}[1]{\chr_{\left\{ #1 \right\}}}
\newcommand{\nrm}[1]{\left\Vert #1 \right\Vert}
\newcommand{\trn}{^\intercal} %
\newcommand{\bx}{\vec{x}}
\newcommand{\bw}{\vec{w}}
\newcommand{\bX}{\vec{X}}
\newcommand{\inv}{^{-1}} %

\newcommand{\calC}{\mathcal{C}}

\newcommand{\calH}{\mathcal{H}}

\newcommand{\calX}{\mathcal{X}}
\newcommand{\calY}{\mathcal{Y}}

\newcommand{\R}{\mathbb{R}}

\newcommand{\beq}{\begin{eqnarray*}}
\newcommand{\eeq}{\end{eqnarray*}}
\newcommand{\beqn}{\begin{eqnarray}}
\newcommand{\eeqn}{\end{eqnarray}}
\newcommand{\paren}[1]{\left( #1 \right)}

\newcommand{\tlprn}[1]{\left\{ #1 \right\}}
\newcommand{\set}[1]{\tlprn{#1}}
\newcommand{\abs}[1]{\left| #1 \right|}

\newcommand{\eps}{\varepsilon}
\newcommand{\trace}{\operatorname{tr}}
\newcommand{\err}{\operatorname{err}}
\newcommand{\sign}{\operatorname{sign}}

\author{
Daniel Berend \\
Ben-Gurion University\\
\texttt{berend@cs.bgu.ac.il} 
\and
Aryeh Kontorovich \\
Ben-Gurion University\\
\texttt{karyeh@cs.bgu.ac.il} \\
\and 
Lev Reyzin \\
U.\ Illinois at Chicago\\
\texttt{lreyzin@uic.edu}
\and
Thomas Robinson\\
Ben-Gurion University\\
\texttt{robinsot@bgu.sc.il} 
}

\title{On Biased Random Walks, Corrupted Intervals, and Learning Under Adversarial Design}
\date{}

\begin{document}

\maketitle

\begin{abstract}
\normalsize
We tackle some fundamental problems in probability theory on corrupted
random processes on the integer line.  We analyze when a biased random walk is expected
to reach its bottommost point and when intervals of integer points can be detected under a natural model of noise.
We apply these results to problems in learning thresholds and intervals under
a new model for learning under adversarial design.

\end{abstract}

\section{Introduction and previous work}

In this paper, we tackle some fundamental questions in probability theory, in particular
by looking at ``corrupted" processes on the integer line.  
We can view biased random walks as instances of such a corrupted process, where
the random walk is ``supposed to'' go up, but is occasionally corrupted and goes down.
Among our results is an analysis of 
when such random walk is expected to hit its bottommost point.  
In the case of intervals, points within and outside the interval are labeled accordingly,
but again, we analyze the case when such labels are corrupted and show when the interval
can be recovered.
 
While these are the main results and are of clear independent interest,
we also connect them to problems in learning under adversarial design. 
%Our results on corrupted processes also have obvious applications in mathematical finance, 
%where, for example, the bottommost
%point of an up-biased random walk corresponds to the point at which 
%a stock index (which
%is often modeled as a biased random walk), starting from a 
%given point, is expected to fall to its
%lowest point. 
Our results are also related to the classical, purely analytical mathematical notion of upcrossing~\citep{Gofer2014,KOOLEN2014144,Teichmann2015}.
%which is motivated by the ``buy low-sell high'' strategies in trading.  
In the case of the integer line processes we find tight estimators on the extremal upcrossings with optimal lower bounds and in the case of the interval we find tight cutoffs for the minimum magnitude of substantial, or true-upcrossings (corresponding to those intervals that can be recovered) as distinguished from transient, or pseudo-upcrossings which are fundamentally indistinguishable from variations which arise merely as an artifact of minor random noise.  

In Section~\ref{sec:rw}, we give both upper and exact bounds on the last time an up-biased random
walk will hit its bottommost point.  In Section~\ref{sec:int} we give results of a similar flavor for
corrupted samples from an interval on the integer line.  Then, in Section~\ref{sec:learning}, we extend
adversarial design to the machine learning and show how our results on corrupted 
processes give learnability results for adversarial design.

\subsection{Random walks}

Much of the literature on random walks on the line focuses on the unbiased case, where at each
time-step, a random walk is equally likely to go up as down.  Quantities such as the expected
distance to the origin are often analyzed, as well as the observation that every point will be
eventually reached by such a walk, which gives rise to phenomena such as ``gambler's ruin''~\citep{harik1999gambler}
and the existence of a few ``favorite sites'' of a random walk~\citep{Toth01},.

%Biased random walks have also been studied, though less extensively.  The probability of the random
%walk exceeding a given value is well understood~\citep{Parry?}. ...

\subsection{Learning under adversarial design}

The celebrated Probably Approximately Correct (PAC) model \citep{DBLP:journals/cacm/Valiant84} has 
been enormously influential in setting the learning-theoretic agenda
over the past thirty years. Indeed, this model has laid the foundation
for a clean and elegant theory while retaining some measure of empirical plausibility.
Regarding the latter criterion, numerous results have aimed at whittling away
at the model's initially somewhat restrictive formulation.
The original requirement of clean labels was relaxed to encompass a benign
type of label noise \citep{DBLP:journals/ml/AngluinL87,DBLP:journals/jacm/Kearns98},
as well as considerably more adversarial noise models \citep{183126,DBLP:journals/ml/KearnsSS94}.
Similarly, the i.i.d.\ sampling assumption --- 
which early learning theory papers often took pains to apologize for
---
has by now been subsumed by far less restrictive mixing conditions
\citep{
gamarnik03,
MR1921877,
london:nips12asalsn,
london:icml13,
mohri-rosta08,
DBLP:journals/jmlr/MohriR10,
rostamizadeh07,
me-cosma-nips,
DBLP:conf/nips/SteinwartC09,
Steinwart2009175,
springerlink:10.1007/s10255-011-0096-4}.
In the {\em online} learning model \citep{MR2409394}, one dispenses with a sampling distribution
entirely, and instead assumes an adversarially chosen sequence of labeled examples. Due to this
model's worst-case nature, one can only prove {\em regret} bounds as opposed to absolute
error estimates.
Many related xresearch directions also include aspects of online learning and non-stationary processes \citep{AgarwalD13,AnavaHMS13,AudiffrenR15,ZiminL17,KuznetsovM15}.

In this section, we propose a distribution-less variant of learning akin to the adversarial design framework for regression. 
Unlike the online setting, 
training data is provided in batch
and we use its structure
to draw conclusions about the range of possible target hypotheses.
In this sense, our learning model conceptually resembles the so-called ``algorithmic luckiness''
framework \citep{DBLP:journals/jmlr/HerbrichW02,DBLP:journals/tit/Shawe-TaylorBWA98},
where the generalization bound depends on the empirical configuration of the training sample
(such as it having a large margin). The salient difference is that the former requires
i.i.d.\ samples, while 
we allow
an arbitrary set of points.
Our model attempts to capture
situations in which the training data is sufficiently informative so as to pick out only
very few potential candidate hypotheses. If, furthermore, all of these candidates are
``close'' in some metric, 
it stands to reason that all of them are 
in fact
close to the target concept in that metric.

Linear regression provides a canonical example 
of this situation
(worked out in detail
in Section~\ref{sec:regress}). Indeed, when the response variable
$y=\bw\trn \vec{x}$
is a noiseless linear
function of the $d$-dimensional predictor variables $\vec{x}$,
it suffices to observe
$d$ labeled points in general position in order
to recover $\bw$ exactly.
When 
the observations are 
perturbed by additive noise
($y=\bw\trn \vec{x}+\xi$),
it will be possible to recover
$\bw$ up to an error that depends on the 
configuration of the
training points as well as the magnitude of the noise (essentially, a signal to noise ratio).

\section{Biased random walks}\label{sec:rw}

Given a random walk on the integer line with upward and downward step probabilities of $p > 1/2$ and $q = 1-p$, respectively, let $T$ be the last time the bottommost point is visited.  
Formally, let $Z_1, Z_2, \ldots$ be independent random variables taking value $1$ with probability $p$ and $-1$ with probability $q=1-p$.  Let $S_0 = 0$ and $S_t = \sum_{i=1}^{t}Z_i$. The bottommost point is a random variable $B = \min_{t \ge 0}\{S_t\}$, and the last time the bottommost value $B$ is visited is $T = \max\{t \ge 0:\ S_t = B\}$.

We would like to understand $\P(T \geq t)$.  
We will find both the probability generating function and the moment generating function of $T$.  
We will also find a closed-form expression for $P(T \geq t)$ in terms of the Gaussian hypergeometric function and an asymptotic expression in terms of the Lerch function.  We recall these function below.  The asymptotic behavior depends on two variables, and because of this we will find two elementary upper bounds for when one or the other variable is dominating the behavior and also one lower bound in elementary terms to indicate the basic level of precision of these estimates.

\subsection{An upper bound on $\P(T\ge t)$}

First, we  show that $\P(T\ge t)=O\paren{(4pq)^{t/2}}$;
in fact, the implied constant may be computed explicitly.

\begin{theorem}
\label{elb1}
Let $T$ be the last time the bottommost point is visited in a random walk with ``up" bias $p$ (and $q= 1-p$), then
$$\P(T \geq t) \leq (1+\sqrt{2})(4pq)^{t/2}.$$
\end{theorem}
We will prove this using the moment generating function which is calculated in Lemma \ref{lem:momgen}.
\begin{lemma}
\label{lem:momgen}
The moment generating function of $T$ is given by
\beq
& & \E [e^{sT}] = \frac{2(p-q)}{1+\sqrt{1-4pqe^{2s}}-2qe^s}, \qquad 0\le s\le\oo2\log\oo{4pq},
\eeq
and in particular,
$$
\E[T] = \frac{q(3-4q)}{(p-q)^2}.
$$
\end{lemma}

Much of the combinatorics of what follows is closely related to the Catalan numbers: 
\begin{align*}
C_n =\frac{1}{n+1}\binom{2n}{n}, \qquad n=0,1,2,\dots
\end{align*}
and their generating function is
\citep[pp. 122]{MR2526440}
\begin{align*}
C(x)= \sum^{\infty}_{n = 0}C_nx^{n} = \frac{1-\sqrt{1-4x}}{2x}.
\end{align*}
and the closely related central binomial coefficient generating function is
\citep[pp. 27-28]{MR2526440}
\begin{align*}
\sum^{\infty}_{n = 0}(n+1)C_nx^{n}=
\sum^{\infty}_{n = 0}\binom{2n}{n}x^{n}
&=\frac{1}{\sqrt{1-4x}}.
\end{align*}

\begin{proof}[Proof of Lemma \ref{lem:momgen}]
It will be instructive to actually start by calculating directly the first two moments,
$\mu_1:=\E[T]$ and $\mu_2:=\E[T^2]$ of $T$.
If on the first step, the drunkard moves upward and never visits $0$ again
(which happens with probability $p-q$), then $T=0$. If he moves upward and returns to $0$ for the first time
in $2k$ steps, which happens with probability $$p\binom{2k-2}{k-1}(pq)^{k-1}q/k=C_{k-1}(pq)^k,$$
then the conditional expectation of $T$ is $2k+\mu_1$. If on the first step he moves down,
the conditional expectation is $\mu_1+1$. 
\noindent Hence:
\beq
\mu_1 &=& (p-q)\cdot0 + pq\sum_{k=1}^\infty C_{k-1} (pq)^{k-1}(\mu_1+2k)  +\ q\cdot(\mu_1+1),
\eeq
so that
\beq
p\mu_1 &=& \mu_1pq\sum_{k=0}^\infty C_{k} (pq)^{k} +
2pq\sum_{k=0}^\infty (k+1)C_k (pq)^{k} + q 
= \oo2 \mu_1\paren{1-\sqrt{1-4pq}} + 2\frac{pq}{\sqrt{1-4pq}} + q,
\eeq
whence
\beq
\mu_1 = \frac{q(3-4q)}{(p-q)^2}.
\eeq
Let us now calculate $\mu_2:=\E[T^2]$. By reasoning similar to above,
\begin{eqnarray*}
\mu_2 &=& (p-q)\cdot0^2 + \sum_{k=1}^\infty \frac{\binom{2k-2}{k-1}}k (pq)^k\cdot\E[(T+2k)^2]  +\ q\cdot\E[(T+1)^2] \\
&=& q\mu_2 + 4\sum_{k=1}^\infty \binom{2k-2}{k-1}(pq)^k x 
+\ 4\sum_{k=1}^\infty k \binom{2k-2}{k-1}(pq)^k + q\mu_2+2qx+q,
\end{eqnarray*}
whence
\beq
\mu_2 = \frac{ q(1-8p+28p^2-16p^3)}{(p-q)^4}.
\eeq 
In this manner we can
compute the 
moment generating function of $T$:
\beq
\E [e^{sT}] 
&=& (p-q)\cdot1+\sum_{k=1}^\infty
C_k (pq)^k\cdot\E [e^{s(T+2k)}] +\ q\cdot\E [e^{s(T+1)}].
\eeq
Thus,
\beq
&& \E [e^{sT}] = \frac{2(p-q)}{1+\sqrt{1-4pqe^{2s}}-2qe^s}, \qquad 0\le s\le\oo2\log\oo{4pq}.
\eeq
\end{proof}

\noindent Now we are ready to prove the main theorem.

\begin{myproof}{Theorem}{\ref{elb1}}
By
Markov's inequality,  
\beq
\P(T\ge t) &\le& 
\frac{2(p-q)e^{-st}}{1+\sqrt{1-4pqe^{2s}}-2qe^s},\quad \text{for all} \quad 0\le s \le \oo2\log\oo{4pq}.
\eeq
The choice $s=\oo2\log\oo{4pq}$ yields
\beq
\P(T\ge t) \quad \le \quad \frac{2(p-q)(4pq)^{t/2}}{1-\sqrt{q/p}}
\quad = \quad 2\sqrt{p}(\sqrt{p}+\sqrt{q})(4pq)^{t/2}.
\eeq
A routine calculation gives
\begin{align*}
\max_{x \in [1/2,1]}\sqrt{x}(\sqrt{x}+\sqrt{1-x})=(1+\sqrt{2})/2.
\end{align*}
Thus
\begin{align*}
\P(T\geq t) \leq (1+\sqrt{2})(4pq)^{t/2},
\end{align*}
which completes the proof.
\hfill$\Box$
\end{myproof}

Because this random walk problem is so natural and of independent interest, in the following section we
give an exact computation for $\P(T \geq t)$.

\subsection{Exact analysis for $\P(T\ge t)$}\label{sec:appA}

Here, we calculate an exact expression for $\P(T \geq t)$ in terms of the Gaussian hypergeometric function
\citep{0898742064}
${}_2F_1$, which is defined for $|x| <1$ by
\begin{align*}
_2F_1(a,b;c;x) =\sum^{\infty}_{n = 0}\frac{(a)_n (b)_n}{(c)_{n}}\frac{x^{n}}{n!},
\end{align*}
where
\begin{align*}
 (\alpha)_n = 
  \begin{cases} 
   1, & \qquad n= 0, \\
   \alpha(\alpha+1)\cdots(\alpha+n-1),      & \qquad  n > 0,
  \end{cases}
\end{align*}
is the rising Pochhammer symbol. 
For convenience we let $\tau=\left \lfloor{\frac{t+1}{2}}\right \rfloor$.

It is elementary to check that
\begin{align}
\label{Cel}
C_n=C_k\frac{(k+(1/2))_{n-k}}{(k+2)_{n-k}}4^{n-k}, \qquad n > k \geq 0.
\end{align}
So
\begin{align*}
\sum^{k-1}_{n=0}C_n x^{n}-C(x)
&=-\sum^{\infty}_{n = k}C_{n}x^{n}
=-x^{k}C_k\sum^{\infty}_{n = k}\frac{(k+(1/2))_{n-k}}{(k+2)_{n-k}}4^{n-k}x^{n-k}.\\
\end{align*}
Therefore
\begin{align}
\label{C1}
\sum^{k-1}_{n=0}C_n x^{n}=
-x^{k}C_{k}\cdot {}_2F_1(1,k+(1/2);k+2;4x)
+\frac{1-\sqrt{1-4x}}{2x}.
\end{align}

It will be convenient to view the random variable $T$ of Lemma \ref{elb1} from a slightly different viewpoint.  Consider the following random walk in $\mathbb{Z}^2$.  Begin at the origin.  If at $(a,b)$, then with probability $p > 1/2$ on the next step move to $(a+1,b-1)$, and with probability $q=1-p$ move to $(a+1,b+1)$.  Let $f(k,l)$ for $0 \leq l \leq k$ be the probability that both $l$ is the absolute maximum second argument reached in the course of the walk and that, among those points with maximum second argument, $k$ is the maximum value of the first argument (thus $T=k$).  So informally $f(k,l)$ is the probability that $(k,l)$ is the last highest point of the walk.  Conditioning on the first step, we get
\begin{align}
\label{cond}
f(k,l)=qf(k-1,l-1)+pf(k-1,l+1), \qquad (k,l) \neq (0,0).
\end{align}
Consider the generating function
\begin{align*}
F(x,y)= \sum^{k}_{l=0}f(k,l)x^{k}y^{l}.
\end{align*}
We
recall (see, e.g., the ``gambler's ruin'' analysis in \citet[Section 17.3.1]{MR3726904})
that the probability of the $(p-q)$-biased random walk never returning to $0$ is $p-q$,
whence
\begin{align}
\label{r}
f(0,0)=p-q.
\end{align}

\begin{theorem}
We have
\begin{align*}
\P(T \geq t)
=&
(p-q)p(pq)^{\tau}
C_{\tau} 
\left({}_2F_1(1,\tau+(1/2);\tau+2;4pq) 
+
4pq
\frac{\tau+(1/2)}{\tau+2}
{}_2F_1(2,\tau+(3/2);\tau+3;4pq) \right).
\end{align*}
\end{theorem}
\begin{proof}
We recall that the Catalan number $C_n$ may be interpreted as the number of walks in our scheme ending at $(0,2n)$ and never passing below the starting point.  With this interpretation, it is easy to see from Equation \ref{r} that
\begin{align}
\label{init}
\sum ^{\infty}_{k =0} f(k,0)x^{k} = (p-q)C\left( pqx^2 \right).
\end{align}
Thus using Equations \ref{cond} and \ref{init} we get
\begin{align*}
F(x,y)&=f(0,0)+
\sum_{k=1}^{\infty}\sum_{l=0}^{k}
f(k,l)x^{k}y^{l}\\
&=f(0,0)+q
\sum_{k=1}^{\infty} \sum_{l=1}^{k}
f(k-1,l-1)x^{k}y^{l}+p
\sum_{k=1}^{\infty}\sum_{l=0}^{k}
f(k-1,l+1)x^{k}y^{l}\\
&=f(0,0)+qxyF(x,y)+p\frac{x}{y}
\sum_{k=1}^{\infty} \sum_{l=0}^{k}
f(k-1,l+1)x^{k-1}y^{l+1}\\
&=f(0,0)+qxyF(x,y)+p\frac{x}{y}
\sum_{l=1}^{k}
f(k,l)x^{k}y^{l}\\
&=f(0,0)+qxyF(x,y)+p\frac{x}{y}F(x,y)-p(p-q)\frac{x}{y}C(pqx^2).
\end{align*}
Therefore
\begin{align*}
F(x,y)=
\frac{p-q}{1-qxy-\frac{px}{y}}
\left(1-\frac{px}{y}C\left(pqx^2\right)\right).
\end{align*}
In particular,
\begin{align*}
F(x,1)(1-x)&=(p-q)(1-pxC(pqx^2)),
\end{align*}
so that
\begin{align*}
F(x,1)&=(p-q)\left(1-pxC\left(pqx^2\right)\right)\left(\sum_{i \geq 0}x^{i}\right),
\end{align*}
which upon multiplying out becomes
\begin{align}
 \label{F1}
F(x,1)&=(p-q)+(p-q)\sum_{k \geq 1}\left( 1- p\sum_{0 \leq n \leq \left \lfloor{\frac{k-1}{2}}\right \rfloor}C_n(pq)^{n}\right)x^{k}.
\end{align}

Let $F(k)$, $k \geq 0$, be the coefficients of the power series expansion of $F(x,1)$ as $F(x,1)=\sum_{k \geq 0}F(k)x^{k}$.  Then $F(0)=p-q$, and from (\ref{F1}), (\ref{Cel}) and (\ref{C1}) we get for $k \geq 1$,
\begin{align*}
F(2k)=F(2k-1)&=
(p-q)\left( 1-p\sum_{0 \leq n \leq k-1}C_{n}(pq)^{n}\right)\\
&=(p-q)\left(1+p(pq)^{k}C_{k}\cdot {}_2F_1(1,k+(1/2);k+2;4pq)
+p\frac{\sqrt{1-4pq}-1}{2pq}\right)\\
&=(p-q)p(pq)^{k}C_{k}\cdot {}_2F_1(1,k+(1/2);k+2;4pq)\\
&=(p-q)p\sum_{n \geq 0}C_{k+n}(pq)^{n+k}.
\end{align*}
We have
\begin{align*}
\sum_{k \geq  \tau}C_{k}x^{k}
=x^{\tau}C_{\tau} \cdot
{}_2F_1(1,k+(1/2);k+2;4x)
\end{align*}
and
\begin{align*}
\sum_{k \geq  \tau}
\left(k- \tau+1\right)
C_{k}x^{k- \tau}
&=
\frac{d}{dx}\left(
\sum_{k \geq  \tau}C_{k}x^{k- \tau+1}
\right)\\
&=
\frac{d}{dx}\left(
x^{1- \tau}
x^{\tau}
C_{\tau} \cdot
{}_2F_1(1,\tau+(1/2);\tau+2;4x)
\right)\\
&=
C_{\tau} \cdot
{}_2F_1(1,\tau+(1/2);\tau+2;4x)
+
4x
C_{\tau}
\frac{\tau+(1/2)}{\tau+2}
{}_2F_1\left(2,\tau+(3/2);\tau+3;4x\right).
\end{align*}
Hence
\begin{align}
\label{sum}
\P(T\geq t)=&(p-q)p\sum_{ k \geq  \tau}
\sum _{n \geq 0}C_{k+n}(pq)^{n+k} \nonumber\\
=&
(p-q)p\sum_{ k \geq \tau}
\left(k-\tau +1\right)
C_{k}(pq)^{k} \nonumber\\
=&(p-q)p(pq)^{\tau}
C_{\tau} \cdot
{}_2F_1(1,\tau+(1/2);\tau+2;4pq) \nonumber\\
&+
(p-q)p(pq)^{\tau}
4pq
C_{\tau}
\frac{\tau+(1/2)}{\tau+2} \cdot{}_2F_1(2,\tau+(3/2);\tau+3;4pq), \nonumber\\
\end{align}
which proves the lemma.
\end{proof}
Since the exact expression in terms of the hypergeometric function is somewhat cumbersome, we obtain an asymptotic expression in terms of 
the Lerch transcendent
\citep{MR2360010}, defined by
\begin{align*}
\Phi (z,s,a)=\sum^{\infty}_{n = 0}\frac{z^{n}}{(n+a)^s} \quad |z| < 1,\quad (a \neq 0, -1,\dots).
\end{align*}

\begin{corollary}
\begin{align}
\label{asym}
\P(T \geq t)
= \Theta \paren{(p-q)(4pq)^{\tau}(\Phi \left(4pq, 3/2,  \tau \right)+\Phi \left(4pq, 1/2,  \tau \right)-\tau\Phi \left(4pq, 3/2,  \tau \right)}).
\end{align}
\end{corollary}
\begin{proof}
We begin with the first equality in equation \ref{sum}
\begin{align*}
\P(T\geq t)&=(p-q)p\sum_{ k \geq  \tau}
\sum _{n \geq 0}C_{k+n}(pq)^{n+k}
\end{align*}
and note that, using Stirling's approximation, we get
\begin{align}
\label{sum2}
\P(T\geq t)&=\Theta \left((p-q)\sum_{ k \geq  \tau}
\sum _{n \geq 0}\frac{(4pq)^{k+n}}{(k+n)^{3/2}}\right).
\end{align}
Summing along diagonals yields
\begin{align*}
\P(T\geq t)&=\Theta \left((p-q)\sum_{ l \geq  \tau}
(l-\tau+1)
\frac{(4pq)^{l}}{l^{3/2}}\right)\\
&=\Theta \left((p-q)\sum_{ m \geq  0}
(m+1)
\frac{(4pq)^{m+\tau}}{(m+\tau)^{3/2}}\right)\\
&=\Theta \left((p-q)(4pq)^\tau \left(
\sum_{ m \geq  0}
(m+\tau)
\frac{(4pq)^{m}}{(m+\tau)^{3/2}}
+(1-\tau)\sum_{ m \geq  0}
\frac{(4pq)^{m}}{(m+\tau)^{3/2}}
\right)\right)\\
&=\Theta \left((p-q)(4pq)^\tau \left(
\Phi(4pq,1/2,\tau)+(1-\tau)\Phi(4pq,3/2,\tau)
\right)\right).
\end{align*}
\end{proof}
We note finally the following further elementary estimates.
\begin{samepage}
\begin{corollary}
We have
\begin{enumerate}
\item $\P(T \geq t) = O\paren{\frac{(4pq)^{(t+1)/2}}{(p-q)^{3}t^{3/2}}}$\label{mainfixednoise}
\item $\P(T \geq t) = \Omega \paren{(p-q)\frac{(4pq)^{(t+1)/2}}{t^{3/2}}}$.
\end{enumerate}
\end{corollary}
\end{samepage}
\begin{proof}
Evaluating the leading term of (\ref{sum2}), we obtain the second estimate.  Equation (\ref{sum2}) also yields
\begin{align*}
\P(T\geq t)&=O \left( (p-q)\sum_{ k \geq  \tau}
\frac{(4pq)^{k}}{k^{3/2}}
\sum _{n \geq 0}(4pq)^{n} \right)
=O\left( \frac{1}{(p-q)}\frac{(4pq)^{\tau}}{\tau^{3/2}}\sum_{ k \geq  0}
(4pq)^{k}\right),
\end{align*}
from which the result follows.
\end{proof}

\section{Corrupted intervals}\label{sec:int}

Let $J=\{1,2,3,\dots,n\}$.  A set of the form $\{m,m+1,m+2,\dots, m+l\} \subseteq J$ is an {\it interval}.  Let $I$ be an interval and $v:J \rightarrow \{\pm 1\}$ be defined by 
\begin{align*}
v(t)=
\left\{
	\begin{array}{ll}
		1, &  t \in I,\\
		-1,\quad &  t \notin I.
	\end{array}
\right.
\end{align*}
Suppose that the values of $v$ are independently corrupted by switching them with probability $q=1-p<1/2$.  Let $\bar{v}$ be the function thus obtained.  We would like to estimate the original interval $I$.  

\subsection{Optimal intervals}

For any interval $I' \subseteq J$, put $\bar{v}(I')=\sum_{i \in I'} \bar{v}(i)$. $I'$ is 
{\it optimal} if, among all 
intervals, it has the highest possible $\bar{v}$ value. 
\begin{theorem}
\label{th:empint}
Let $T=\max (|I \triangle I'|: I' \, \text{is optimal})$. 
There exists a constant $C>0$ such that
\begin{align*}
\P (T \geq t)\ \leq\ C 
\frac{n(4pq)^{t/2}}{(p-q)^4},
\qquad n \geq 1, \, 0 \leq t \leq n,\, 1/2 < p \leq 1.
\end{align*}
\end{theorem}
\begin{proof}
By Theorem \ref{elb1}, for any particular choice of $I'$ with $|I' \triangle I|=i$, there exists a constant $C_1$ such that $\P(\bar{v} (I') \geq \bar{v}(I)) \leq C_1 (4pq)^{i/2}$.  Obviously, for fixed $i$, the number $N_1$ of those choices $I'$ which do not intersect $I$ satisfies $N_1 \leq n$.  It is also easy to see that the number $N_2$ of those which do intersect satisfies $N_2 \leq 4i$.  Therefore, by a union bound we get
\begin{align*}
\P (T \geq t) \leq 4C_1 
 \sum _{i=t}^{\infty} (n+i)(4pq)^{i/2} \leq C_2 \left( n\frac{(4pq)^{t/2}}{1-\sqrt{4pq}}+\frac{t(4pq)^{t/2}}{(1-\sqrt{4pq})^2}\right),
\end{align*}
where $C_2$ is some constant.  The result follows easily.
\end{proof}

\subsection{Phantom intervials}

Theorem \ref{th:empint} does not tell the whole story.  Consider, for example, $$I=[1,3m], \hat{I}_1=[2m+1,5m],\hat{I}_2=[3m+1,4m],\hat{I}_3=[m^2+1,m^2+m],$$ where $m$ is large and all the candidates $\hat{I}_1,\hat{I}_2,\hat{I}_3$ are optimal.  Our measure of deviation of a candidate from the correct interval, namely the size of the symmetric difference of the two, indicates that $\hat{I}_1,\hat{I}_2,\hat{I}_3$ are all equally good estimates.  Naturally, however, we view $\hat{I}_3$ as much worse than $\hat{I}_1$ and $\hat{I}_2$.  This motivates the following definition.  An optimal guess, not intersecting the original interval $|I|$, is a {\it phantom}.  Whether or not a phantom is likely to appear depends crucially on how big $I$ is relative to $J$; if $I$ is small then a phantom is likely, while if it is large then a phantom is unlikely. Theorem \ref{ghmain} makes this precise.  The case where $p=1$ is trivial, of course, and from here until the end of the section we assume that $1/2 <p < 1$.  To state the theorem, we recall the relative entropy function $D(\cdot||\cdot)$, defined by
\begin{align*}
D(x||y)=x\log (x/y)+(1-x)\log ((1-x)/(1-y)), \qquad 0<x,y<1.
\end{align*}

We denote by $G$ the event that there exists a phantom.
\begin{theorem}
\label{ghmain}
For every fixed $\varepsilon >0$:
\begin{align*}
\lim_{n \rightarrow \infty} \P(G)=
\begin{cases}
0, & \qquad  |I| > \left(\frac{1}{D(0.5||p)}+\varepsilon\right) \log n,\\
1, & \qquad   |I| < \left(\frac{1}{D(q||p)}-\varepsilon \right)\log n.
\end{cases}
\end{align*}
\end{theorem}

%The proof of this theorem is quite technical and appears in full in the following section.

In order to prove Theorem  we will use some well-known results, as well as several further lemmas.  We begin by recalling  the following elementary property of relative entropy:

\begin{proposition}
\label{Di}
For arbitrary fixed $0<y_0<1$, the function $D(x||y_0)$ is decreasing for $0<x<y_0$ and increasing for $y_0<x<1$.  
\end{proposition}

\begin{lemma}
\label{aa}
$
\displaystyle 
\P(G) \leq 
\frac{1-e^{-(n+1)D(1/2 || p)}}{1-e^{-D(1/2 || p)}}
ne^{-|I|D(1/2 ||p)}$.
\end{lemma}
In the proof we will use a Chernoff-Hoeffding bound.
Let $B(n',p')$ denote the binomial distribution, where $n'$ is the number of trials and $p'$ the probability of success in each.  The inequality states that for binomial variables $X \sim B(n',p')$:
 
\begin{align}
\label{bintaillower}
\P(X \leq k)& \leq e^{-n'D\left(\left(\frac{k}{n'}\right) ||p' \right)}, \qquad  0 < \frac{k}{n'}<p'.
\end{align}
\begin{proof}
Consider a potential phantom candidate $I'$.  Let $S_1$ be the number of elements of $I$ which switch and $U_1$ the number of those that do not.  Also, let $S_2$ be the number of elements of $I'$ which switch and $U_2$ the number of those that do not.  Then, if $I'$ is to be a phantom candidate, we must have
\begin{align*}
v(I)=U_1-S_1 \leq S_2-U_2=v(I'),
\end{align*}
namely
\begin{align*}
U_1+U_2 \leq S_1+S_2.
\end{align*}
Let $U=U_1+U_2$ and $S=S_1+S_2$, so that
\begin{align*}
U+S&=|I|+|I'|.
\end{align*}
We want an upper bound for the probability that $U \leq S$, which is equivalent to

\begin{align*}
U \leq (|I|+|I'|)/2,
\end{align*}
so that Chernoff-Hoeffding's bound (\ref{bintaillower}) yields
\begin{align}
\label{gh1}
\P(G(I'))
\leq 
\P(U \leq  (|I|+|I'|)/2)
 \leq e^{-(|I|+|I'|)D((1/2)||p)},
\end{align}
where $G(I')$ is the event that $I'$ is a phantom. 
Therefore, using a union bound, we have that
\begin{align*}
\P(G) \leq 
\sum_{i=0}^{n}ne^{(-|I|-i)D((1/2)||p)},
\end{align*}
from which the result follows.  
\end{proof}

We say that $I'$ is {\it overlapping $t$-distant} from $I$ if it satisfies the two conditions: 
\begin{enumerate}
\item $I' \cap I \neq \emptyset$.
\item $|I' \triangle I| \geq t$.
\end{enumerate}
\begin{lemma}
\label{overl}
The probability $P$ that there exists an optimal overlapping $t$-distant interval $I'$  satisfies
\begin{align*}
P =O\left(t(4pq)^{t/2}\right).
\end{align*}
\end{lemma}
\begin{proof}
By the proof of Theorem \ref{elb1}, for any overlapping $i$-distant interval $I'$,

the probability $P_1$ that $I'$ is optimal satisfies
\begin{align*}
P_1 = O\left( (4pq)^{i/2}\right).
\end{align*}
Moreover, there are at most $4i$ possible intervals satisfying these conditions, so that by a union bound we get
\begin{align*}
P &=O\left( \sum _{i \geq t}i(4pq)^{i/2}\right)= O 
\left( \frac{t(4pq)^{t/2}}{(1-\sqrt{4pq})}+\frac{(4pq)^{t/2}}{(1-\sqrt{4pq})^2}
\right),\\
\end{align*}
which immediately gives the result.
\end{proof}

For a non-negative $r$, a {\it weight}-$r$ {\it pseudo-phantom} is an interval $I'$ satisfying 
$I' \cap I = \emptyset$, such that $v(I') \geq r$.  We denote by $G(r)$ the event that there exists a pseudo-phantom of weight $r$, and by $G(I',r)$ the event that $I'$ is such a pseudo-phantom.

\begin{lemma}
\label{pseudo-phantom} 
For $r \leq |I|$, 
\begin{align*}
\P(G(r)) \geq
 1-
\left(1-
\frac{1}{\sqrt{2|I|}}e^{-|I|D\left(\frac{r+|I|}{2|I|}||q\right)}
\right)^{ \frac{n}{|I|}-4}.
\end{align*}
\end{lemma}

In the proof we will use the following bound for $X \sim B(n',p')$, which follows readily from Stirling's formula:
\begin{align}
\label{st}
\P(X \geq k)& \geq \frac{1}{\sqrt{2n'}}e^{-n'D\left(\left(\frac{k}{n'}\right) ||p' \right)}, \qquad  p' < \frac{k}{n'}<1.
\end{align}
\begin{proof}
Consider possible pseudo-phantoms $I'$ with $|I'|=|I|$.
Set $\beta = \lfloor \frac{n}{|I|}\rfloor -3$.  We can find $\beta$ pairwise disjoint intervals $I'_1, I'_2, \dots ,I'_{\beta} \subseteq \{1,2, \dots,n \}-I$.  We bound $\P(G(r))$ from below by the probability that at least one of $I'_1, I'_2, \dots ,I'_{\beta}$ is a weight-$r$ pseudo-phantom.  
 
Denote by $BG(r)$ the event that at least one of $I'_1, I'_2, \dots ,I'_{\beta}$ is a pseudo-phantom of weight $r$ in at least one bracket.  Denote by $B(r)$ the event that $I'_1$ is a pseudo-phantom of weight $r$.  Since the events that the $I_j$'s are weight-$r$
pseudo-phantoms are independent and equi-probable, (\ref{st}) yields:
\begin{align*}
\P(G(r))&\geq \P(BG(r))\\
&= 1-\P(\neg BG(r))\\
&=1-\P( \neg B(r))^{\beta}\\
&\geq 1-
\left(1-
\frac{1}{\sqrt{2|I|}}e^{-|I|D\left(\frac{r+|I|}{2|I|}||q\right)}
\right)^{\beta}\\
&= 1-
\left(1-
\frac{1}{\sqrt{2|I|}}e^{-|I|D\left(\frac{r+|I|}{2|I|}||q\right)}
\right)^{\lfloor \frac{n}{|I|}\rfloor-3}\\
&\geq 1-
\left(1-
\frac{1}{\sqrt{2|I|}}e^{-|I|D\left(\frac{r+|I|}{2|I|}||q\right)}
\right)^{ \frac{n}{|I|}-4}.
\end{align*}
\end{proof}

\begin{myproof}{Theorem}{\ref{ghmain}}
The first case follows readily from Lemma \ref{aa}.  

Let $OD(t)$ denote the event that there exists an optimal overlapping $t$-distant interval.  Since $|I| < \left(\frac{1}{D(q||p)}-\varepsilon \right)\log n$, we may write $|I|=c \log n$ for some $c < 1/D(p+\delta||q)$, where $0<\delta \leq q.$
By a union bound,
\begin{align*}
\P[G(p-q+2\delta)|I|)] \leq \P[G]+
\P[\bar{v}(I) \geq (p-q+ \delta)|I|]
 +\P[OD(\delta|I|)],
\end{align*}
or
\begin{align}
\label{aleph}
\P[ G ] \geq 
\P[G((p-q+2\delta)|I|)]-
\P[\bar{v}(I) \geq (p-q+ \delta)|I|]
 -\P[OD(\delta|I|)].
\end{align}
By Lemma \ref{pseudo-phantom}, 
\begin{align*}
\P[G((p-q+2\delta)|I|)] \geq 
1-
\left(1-
\frac{1}{\sqrt{2|I|}}e^{-|I|D\left(\frac{p+\delta}{2|I|}||q\right)}
\right)^{ \frac{n}{|I|}-4}.
\end{align*}
Now,
\begin{align*}
\lim _ {n \rightarrow \infty}
\log \left(
1-\frac{1}{\sqrt{2\lfloor c \log {n}\rfloor}}
e^{-\lfloor c D(p+\delta||q)\log n \rfloor }
\right)^{\frac{n}{\lfloor c \log n \rfloor} -4}
&\leq
\lim _ {n \rightarrow \infty}
\log \left(
1-\frac{1}{\sqrt{2c \log {n}}}
e^{-c D(p+\delta||q)\log n }
\right)^{\frac{n}{c \log n} -4}\\
&=
\lim _ {n \rightarrow \infty}
\frac{n}{c \log n} \cdot
\log \left(
1-\frac{1}{\sqrt{2c \log {n}}}
e^{-c D(p+\delta||q) \log n}
\right)\\
&\leq
\lim _ {n \rightarrow \infty}
\frac{n}{c \log n} \cdot
\frac{-1}{\sqrt{2c \log {n}}}
e^{-c D(p+\delta||q)\log n }\\
&=
\lim _ {n \rightarrow \infty} \frac{-1}{\sqrt{2c^3 \log ^3 n}}n^{1-cD(p+\delta||q)}\\
&= -\infty.
\end{align*}
Therefore
\begin{align}
\label{bet}
\lim  _ {n \rightarrow \infty} \P(G((p-q+2\delta)|I|)) =1.
\end{align}
By Chernoff's inequality
\begin{align}
\label{gimmel}
\P(\bar{v}(I) \geq (p-q+ \delta)|I|) 
\leq e^{-|I|D\left( p+\frac{\delta}{2}||p\right)}
\underset{n \rightarrow \infty}{\longrightarrow}0.
\end{align}
By Lemma \ref{overl}
\begin{align}
\label{daled}
\lim  _ {n \rightarrow \infty} \P(OD(\delta|I|) =0.
\end{align}
The second case in the theorem follows from (\ref{aleph}), (\ref{bet}), (\ref{gimmel}) and (\ref{daled}).
\hfill$\Box$
\end{myproof}

\section{An application to learning under adversarial design}\label{sec:learning}

%\subsection{Model}
We define a learning model which we call 
\textbf{Approximately Correct Learning under Adversarial Design}. 
This model tries to capture the phenomenon that, when 
learning a restricted class of hypotheses, 
it is often the case that a few
arbitrarily chosen noiseless
examples pin down the target function uniquely. 
We begin, as in the PAC model, with an instance space $\X$,
a label space $\Y$,
and a hypothesis class $\H\subseteq\Y^\X$.
Unlike PAC, however, there is no distribution over $\X$
from which a training sample would be drawn; instead,
some arbitrary data set $S\subseteq\X$ is provided.
A teacher chooses a target concept $c\in\H$
and labels every $x\in S$ with its true label $c(x)$.
These labels are then corrupted by a noise process $\eta$,
and the learner ultimately receives the set of pairs $(x,y)$,
with $x\in S$ and $y=\eta(c(x))$.\footnote{
The noise processes in this paper will be specified by a single parameter,
which we will also denote by $\eta\in\R$. No confusion should arise.}
To finalize our specification of a learning problem under adversarial design, we need a {\em loss} function
$\ell:\H\times\H\to[0,\infty)$
over the hypotheses.
The learner observes the labeled data $\set{(x_i,y_i): x_i\in S}$
and produces a hypothesis $h\in\H$. This induces the random variable
$L=\ell(h,c)$ --- where the only source of randomness is the label noise process.
We say that the quadruple $(\H,S,\eta,\ell)$ is $(\eps,\delta)$-{\em learnable under adversarial design}
if $\P(L>\eps)<\delta$.

We initiate the study of  learnability under adversarial design by giving a positive result for the
concept class of thresholds
$h_a:x\mapsto \pred{x\ge a}$
indexed by $a\in\R$,
where the noise process flips a label with probability $\eta<1/2$
and $\ell(a,a')=\abs{a-a'}$.
While this target class is rather simple, its analysis already turns out to be
nontrivial.
In Theorem~\ref{thm:main},
we show that,
as long as the training data contains $$\Omega\paren{\log(1/\delta)/(1-2\eta)^2}$$
points within a distance of $\eps$ from the target threshold,
the $(\eps,\delta)$ adversarial design learnability condition is satisfied by an ERM learning algorithm.
We further show in Corollary~\ref{cor:pac} that this recovers
a known noisy PAC learnability result for thresholds.
We then turn our attention to intervals and characterize the structural conditions
necessary for the ERM learner to find the target interval.

\subsection{Warm-up: regression}
\label{sec:regress}
In this section, we 
use linear regression as a vehicle for building intuition regarding the model of learnability under adversarial design.
The simplest case is one-dimensional:
an affine
function determines the $y$ value from the $x$ coordinate of a point.  Notice 
that if there is no noise in the $y$ values, any two distinct points will exactly determine the line.
If the $y$ coordinate is 
corrupted by additive Gaussian noise,
two ``far'' points
are more informative than two close ones, and more points are more informative than fewer.
With this in mind, let us consider the general $d$-dimensional case.
Let us arrange our data $S=\set{\bx^1,\ldots,\bx^m}\subset\R^d$
as an $m\times d$ matrix $\bX$,
and assume that $\bX\trn\bX$ is non-singular.
The learner's hypothesis $\emp{\bw}$ will be the ordinary least squares (OLS) estimate:
\beq
\emp{\vec w} = (\vec{X}\trn\vec{X})\inv \vec{X}\trn \vec y,
\eeq
where $\vec y\in \R^m$ is the ``labels'' vector.
We assume Gaussian white label noise,
\beq
\vec y = \vec{X}\tru{\vec{w}}+\vec{\xi},
\eeq
where $\vec\xi\sim N(0,\eta^2 \vec{I_m})$ for a certain noise parameter $\eta>0$.
Thus,
\beq
\emp{\vec w} &=& (\vec{X}\trn\vec{X})\inv \vec{X}\trn \vec y \\
&=& (\vec{X}\trn\vec{X})\inv \vec{X}\trn (\vec{X}\tru{\vec{w}}+\vec{\xi})\\
&=& (\vec{X}\trn\vec{X})\inv (\vec{X}\trn \vec{X})\tru{\vec{w}}
+(\vec{X}\trn\vec{X})\inv \vec{X}\trn \vec{\xi}\\
&=&
\tru{\vec{w}} 
+(\vec{X}\trn\vec{X})\inv \vec{X}\trn \vec{\xi}.
\eeq
Hence, the error vector $\vec z:=\emp{\vec{w}}-\tru{\vec{w}}$
is distributed
as $N(0,\eta^2 \vec{B}\vec{B\trn})$, where $\vec{B}=(\vec{X}\trn\vec{X})\inv \vec{X}\trn $.
Simplifying $\vec{B}\vec{B\trn} = (\vec{X}\trn\vec{X})\inv$, this yields
\beq
\vec z
\sim N(0,\eta^2 (\vec{X}\trn\vec{X})\inv),
\eeq
and hence
\beq
\E\nrm{\vec z}_2^2 = \sum_{i=1}^d \E z_i^2 
= \eta^2 \sum_{i=1}^d [(\vec{X}\trn\vec{X})\inv]_{ii} 
= \eta^2\trace((\vec{X}\trn\vec{X})\inv) 
= \eta^2\sum_{i=1}^d \sigma_i^{-2}(\vec X),
\eeq
where $\sigma_i(\bX)$ is the $i$th singular value.

Let us first make the connection to classical statistics, which makes additional assumptions on $\bX$.
In the {\em random design} setting, the data points $\bx^i$ are assumed to be sampled from some distribution.
In the simple case where $\bx^i\sim N(0,I_d)$ and $m\gg d$, 
all of the singular values 
$\sigma_i(\bX)$ are of order of magnitude
$\sqrt m$ \citep{1183.15031}, which yields the estimate
\beqn
\label{eq:regr}
\E\nrm{\emp{\bw}-\tru{\bw}}_2^2 
=O\paren{
\frac{d\eta^2}{m}
}.
\eeqn
Analogous estimates hold in typical {\em fixed design} settings \citep{MR2013911}.
An adversarial design result is also readily obtained from (\ref{eq:regr}). 
If the $m\times d$ data matrix $\bX$
satisfies $m\ge d$ and $\sigma_d=\Omega(\sqrt m)$,
then Markov's inequality applied to (\ref{eq:regr})
yields an $(\eps,\delta)$ learnability under adversarial design learnability result for linear regression,
with $\eps=O(d\eta^2/\sqrt m),\ \delta=O\left(\oo{\sqrt m}\right),$ and the loss function
$\ell(\bw,\bw')=\nrm{\bw-\bw'}_2^2.$

\subsection{Learning thresholds and intervals under adversarial design}

Consider the class $\H$ of thresholds over $\R$.  
Each function $h_a \in \H$ can be represented by a scalar value $a \in \R$ and 
assigns the positive label to all points in $[a,\infty)$.
Perhaps the most natural learner for this problem is one that chooses an ERM hypothesis $\hat a$ so as to minimize
the number of mistakes on the finite sample sample.

In this section, we will prove the following theorem.
\begin{theorem}\label{thm:main}
Let $h_{a^*}$ be the target threshold.
There exists an $m_0=m_0(\delta,\eta)$, of magnitude
\beqn
\label{eq:sharper}
m_0 &=&
O\paren{\frac{\log{{({1}/{\delta})}}}{\log{({1}/{(\eta - \eta^2)})}}} \\
&\subset &
\label{eq:looser}
 O\paren{\frac{ \log{{(1/\delta)}}}{(1-2\eta)^2}},
\eeqn
such that if the sample
contains $m\ge m_0$
data points both in the interval $[a^*-\eps,a^*)$ and also in $(a^*,a^*+\eps]$, then
any 
ERM learner will output a hypothesis $h_{\hat{a}}$ such that $|a^* - \hat{a}|  \le \eps$ 
with probability at least $ 1- \delta$.
\end{theorem}
\begin{proof}
First 
assume our sample consists of $n$ consecutive integers to both
the left and right of the threshold.
We will now analyze the given ERM classifier.  
Consider the question: when would the ERM classifier
choose an integer value $\hat{a}  > a^*$?  
This can happen only if the number of negative examples exceeds
the number of positive examples in $[a^*,\hat{a}]$.  This event can be analyzed 
from the viewpoint of a random walk in 1 dimension, starting at $0$ and moving ``up'' by $1$ 
upon seeing a positively labeled point
and ``down'' by $1$ upon seeing a negatively labeled point, which will happen for each point
 with probability $1-\eta$ and $\eta$, respectively.  It is easy to see that the
event whereby the
 ERM value at $\ge a^*$ will happen at time $\hat{a}$ where the drunkard has reached his
 bottommost point, as illustrated in Figure~\ref{fig:walk}.
 
\begin{figure}
\begin{center}
\includegraphics[scale=0.5]{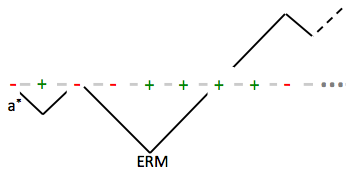}
\end{center}
\caption{An illustration of the ERM learner as an extreme point of a random walk.  Only
points to the r.h.s.\ of $a^*$ are depicted.  Positive points are correctly labeled; the labels
of negative points are due to noise.}
\label{fig:walk}
\end{figure}
Now we will analyze the ``deviation'' random variable $D:=|\hat a-a^*|$.
(In case of non-unique $\hat a$, define $D$ to be the ``worst'' deviation.)
Let $D_+$ be the distance from $a^*$ to the farthest empirical optimum to its right and define $D_-$ 
analogously on the left.
Clearly, $D_+,D_-$ are independent and identically distributed, and $D=\max(D_+,D_-)$. Hence
$$ \E[D]\le 2\E[D_+].$$
We will define the random variable $T$ as the {\bf last} time that the drunkard visits the bottommost point (i.e., minimum)
of his entire walk. 
Note that $T$ and $D_+$ have the same distribution.
Now, as $n \rightarrow \infty$, we use Theorem~\ref{elb1} and substitute $4pq = 4\eta-4\eta^2$, 
and we get that with probability $1-\delta$,
the ERM algorithm will produce a hypothesis $\hat{a}$ such that\footnote{
The 
inequality
$
{\frac{
1
}{\log{({1}/{(4\eta - 4\eta^2)})}}} 
\le
{\frac{1}
{(1-2\eta)^2}}
$
follows from the elementary estimate
$
1-x \le \log\oo x, \mathrm{for}\ x>0,
$
applied to $x=4\eta-4\eta^2$.
It
was mainly invoked to obtain a bound in a form familiar for comparing to PAC under 
classification noise.
Observe, however, that as $\eta\to0$, the two bounds 
(\ref{eq:sharper})
and
(\ref{eq:looser})
become qualitatively different.
In this regime, the 
estimate in (\ref{eq:sharper})
becomes $O(1)$ --- which makes sense, since without noise,
a single pair of sample points trapping the target threshold within $\eps$ from left and right
suffices to achieve the desired accuracy.
In contradistinction, even for $\eta=0$, the 
estimate in (\ref{eq:looser}) remains of order $\log(1/\delta)$.}

\begin{eqnarray*}
|a^* - \hat{a}| &\;\in\;&  
O\paren{\frac{\log{{({1}/{\delta})}}}{\log{({1}/{(\eta - \eta^2)})}}} 
\;\subset \;
 O\paren{\frac{ \log{(1/\delta)}}{(1-2\eta)^2}}.
\end{eqnarray*}
Finally, we can get rid of our assumption that the data lies on arbitrarily many integer points on the line
by making the following two observations:
\begin{enumerate}
\item Our argument does not require the sample to be on integer points.
Rather, the ``drunkard'' takes a step upon encountering a new point, so as long as he 
sees  $$m = O\paren{\frac{\log{{(1/\delta)}}}{\log(1/(\eta - \eta^2))}}$$ data points within $\eps$
from $a^*$, the ERM  hypothesis will also be within distance~$\eps$.
\item Given that the algorithm has seen sufficiently many points (denoted $m$ above) within
$\eps$ of $a^*$, seeing additional data is not necessary for the algorithm to succeed (in
fact, it only increases the probability of failure).  Hence, seeing only $m$ samples is sufficient.
\end{enumerate}
This completes  the proof.
\end{proof}

To analyze the case of an interval, we must again compare the ERM hypothesis with the true interval.
The size overlap between the ERM and the target will guide how many points need to be within $\epsilon$
of both thresholds.  As before, focus on analyzing an idealized problem, where the interval simply contains points from a set.
Theorem~\ref{th:empint} gives a bound on the number of points needed.

\begin{corollary}[To Theorem~\ref{th:empint}]
Let $h_{(a^*,b^*)}$ be the target interval.
There exists an $m_0=m_0(\delta,\eta)$, of magnitude
$$
m_0 =
\tilde{O}\paren{\frac{\log{{({1}/{\delta})}}}{\log{({1}/{(\eta - \eta^2)})}}} 
\subset 
 \tilde{O}\paren{\frac{ \log{{(1/\delta)}}}{(1-2\eta)^2}},
$$
such that if the sample
contains $m\ge m_0$
data points both in the intervals $[a^*-\eps,a^*)$, $(a^*,a^*+\eps]$, $[b^*-\eps,b^*)$, and
$(b^*,b^*+\eps]$ then
any 
ERM learner will output a hypothesis $h_{(\hat{a},\hat{b})}$ such that both $|a^* - \hat{a}|  \le \eps$ and
$|b^* - \hat{b}|  \le \eps$ hold
with probability at least $ 1- \delta$.
\end{corollary}
\begin{proof}
We follow the logic of the proof of Theorem~\ref{thm:main}, but using the bound from Theorem~\ref{th:empint}
instead of Theorem~\ref{elb1}.  We again substitute $4pq = 4\eta-4\eta^2$ (as $q=\eta$ and $p=1-\eta$), which cancels the
$(p-q)^4=1$.  The remaining difference, the term of $n$, only a polylog factor to the analysis of the sample complexity.
\end{proof}

\subsection{Relationship to PAC learning}
Let us briefly recall the (proper\footnote{
The qualifier ``proper'' means that the teacher and learner both work with the same concept class 
$\C$.
More generally, the learner might choose to produce hypotheses $h\in\H\supsetneq\C$
from a strictly richer collection,
but we will not consider this ``improper'' setting here.
}) noisy PAC learning model \citep{DBLP:journals/ml/AngluinL87,Kearns97}.
A teacher and a learner agree on the {\em instance space} $\X$ and concept class
$\C\subset\set{0,1}^\X$.
The teacher privately chooses any $c\in\C$ and any distribution $D$ over $\X$.
He proceeds to draw $m$ examples i.i.d.\ $\sim D$ and label each example $x_i$ with $c(x_i)$.
Each label is then flipped independently with probability $\eta<1/2$, and the learner gets to
see the $m$ examples along with their (potentially corrupted) labels.
Based on this noisy sample, the learner produces a hypothesis $h\in\C$, with its
associated {\em generalization error} $$\err(h):=\P_{X\sim D}[h(X)\neq c(X)].$$
The learner is $(\eps,\delta)$-successful if 
$$\P(\err(h)>\eps)<\delta.$$
A PAC learner is one who is $(\eps,\delta)$-successful whenever $m>m_0$,
where $m_0$ depends on 
$\C,\eps,\delta$ but, crucially, not $c$ or $D$.

We will now show that the  adversarial design result we proved for thresholds in Theorem~\ref{thm:main}
has implications for noisy PAC learnability of this concept class under the uniform distribution.
Formally, we take $\X=[0,1]$ with uniform distribution and $\C$ to be the collection of thresholds
$c_a$, as defined above.
\begin{corollary}
\label{cor:pac}
Thresholds are (properly) PAC learnable, with label noise $0\le \eta<1/2$, under
the uniform distribution over $[0,1]$, by any  
ERM algorithm that has access $m\ge m_0$
i.i.d. examples, for
$$m_0=O\paren{\frac{\log^2{(1/\delta)}}{\eps(1-2\eta)^2}}.$$
\end{corollary}
\begin{proof}
Theorem~\ref{thm:main} says that we require  
$m = \Omega\paren{\frac{ \log{(1/\delta)}}{(1-2\eta)^2}}$
data points
within $\eps$ to the left and right of the target, $a^*$.  
This will happen after $O(\log(1/\delta)m/\eps)$
points are drawn from the uniform distribution, giving a sample complexity of 
$O\paren{\frac{ \log^2{(1/\delta)}}{\eps(1-2\eta)^2}}$
and
showing the PAC learnability of thresholds under the uniform distribution with an
ERM algorithm.  Note that in the event that the target threshold $a^*$ lies within $\eps$ of a boundary ($0$ or $1$),
we only need points to one side of the threshold, and the same analysis goes through.
\end{proof}

%add reviewer objections

\subsection{Discussion}\label{sec:disc}

While our random walk techniques gave a nice illustration of their usefulness via an 
application to learning, it is possible
to get these learning guarantees more simply.
To obtain our asymptotic learning results in 
Theorem~\ref{thm:main} and Corollary~\ref{cor:pac},
to get learnability results for thresholds and intervals,
we only cared about the order of the points, not the underlying distribution.  
Hence, these results can alternatively be obtained by reducing toPAC learning under the uniform distribution over the unit interval.
Combining results for random classifciation noise (or the noise condition of~\cite{MassartN06}) with
VC bounds, we can cocnlude that 
$$\P( h_{\hat{a}}(x) \neq h_{a*}(x) ) \le O\left(  \frac{\log1/\delta} {m (1-2\eta)^2}\right),$$
which is sufficient for the purposes of these theorems.  This analysis, however is not general
and leads us to the following open problem.

\textbf{Open problem:} 
A natural high-dimensional analogue of thresholds are
half-spaces: $$\H=\set{ \bx \mapsto \sign(\bw\cdot\bx): \bw\in\R^d}.$$
For some $p\ge1$,  define the loss 
\begin{eqnarray*}
\ell_p(\bw,\bw') &=& \nrm{\bw-\bw'}_p
= \paren{\sum_{i=1}^d|w_i-w_i'|^p}^{1/p}.
\end{eqnarray*}
The noise process is the same as for the thresholds: each label is flipped with probability $0 \le \eta < 1/2$.
What non-trivial property must the training data satisfy in order to assure  learnability under adversarial design?
One possibly helpful fact is that a half-space also imposes an ``ordering" (similar
to the ordering implicitly used by our threshold analysis) on the points along the normal
to its hyperplane, 
and that $n$ points in $d$ dimensions only admit $O(n^d)$ different orderings when projected onto 
lines~\citep{cover1967number}.

\section*{Acknowledgements}

We thank an anonymous reviewer of a previous version of this paper for pointing us to the
results that led to the discussion at the beginning of Section~\ref{sec:disc}.

Daniel Berend was supported in part by the Milken
Families Foundation Chair in Mathematics and the
Cyber Security Research Center at Ben-Gurion University.
Aryeh Kontorovich 
was supported in part by
Israel Science Foundation
grant 1602/19 and by Google Research.
Lev Reyzin was supported
in part by grants CCF-1934915 and CCF-1848966 from the National Science Foundation.
Thomas Robinson was supported 
in part by 
Israeli Science Foundation grant 1002/14 and the Cyber Security Research Center at Ben-Gurion University

\bibliographystyle{plainnat}
\bibliography{paper}

\end{document}